\documentclass[english]{article}
\usepackage[T1]{fontenc}
\usepackage[latin9]{inputenc}
\usepackage{geometry}
\usepackage{babel}
\usepackage{refstyle}
\usepackage{units}
\usepackage{mathtools}
\usepackage{amsmath}
\usepackage{amsthm}
\usepackage{amssymb}
\usepackage[round,authoryear]{natbib}
\usepackage[unicode=true,pdfusetitle,
 bookmarks=true,bookmarksnumbered=false,bookmarksopen=false,
 breaklinks=false,pdfborder={0 0 1},backref=false,colorlinks=false]
 {hyperref}

\makeatletter

\AtBeginDocument{\providecommand\lemref[1]{\ref{lem:#1}}}
\AtBeginDocument{\providecommand\thmref[1]{\ref{thm:#1}}}
\AtBeginDocument{\providecommand\propref[1]{\ref{prop:#1}}}
\AtBeginDocument{\providecommand\corref[1]{\ref{cor:#1}}}
\RS@ifundefined{subsecref}
  {\newref{subsec}{name = \RSsectxt}}
  {}
\RS@ifundefined{thmref}
  {\def\RSthmtxt{theorem~}\newref{thm}{name = \RSthmtxt}}
  {}
\RS@ifundefined{lemref}
  {\def\RSlemtxt{lemma~}\newref{lem}{name = \RSlemtxt}}
  {}

\theoremstyle{plain}
\newtheorem{thm}{\protect\theoremname}
  \theoremstyle{definition}
  \newtheorem{defn}[thm]{\protect\definitionname}
  \theoremstyle{plain}
  \newtheorem{prop}[thm]{\protect\propositionname}
  \theoremstyle{plain}
  \newtheorem{lem}[thm]{\protect\lemmaname}
  \theoremstyle{plain}
  \newtheorem{cor}[thm]{\protect\corollaryname}

\usepackage{authblk}

\let\ref\autoref

\author[1,2,3]{Weinan E}
\author[4]{Qingcan Wang}
\affil[1]{Department of Mathematics and PACM, Princeton University}
\affil[2]{Center for Big Data Research, Peking University}
\affil[3]{Beijing Institute of Big Data Research}
\affil[4]{PACM, Princeton University}

\date{}

\makeatother

  \providecommand{\corollaryname}{Corollary}
  \providecommand{\definitionname}{Definition}
  \providecommand{\lemmaname}{Lemma}
  \providecommand{\propositionname}{Proposition}
\providecommand{\theoremname}{Theorem}

\begin{document}

\title{Exponential Convergence of the Deep Neural Network Approximation
for Analytic Functions}

\maketitle
\vspace{-2em}
\begin{center}
\emph{Dedicated to Professor Daqian Li on the occasion of his 80th birthday}
\end{center}
\vspace{1em}
\begin{abstract}
We prove that for analytic functions in low dimension, the convergence
rate of the deep neural network approximation is exponential.
\end{abstract}

\section{Introduction}

The approximation properties of deep neural network models is among
the most tantalizing problems in machine learning. It is widely believed
that deep neural network models are more accurate than shallow ones.
Yet convincing theoretical support for such a speculation is still
lacking. Existing work on the superiority of the deep neural network
models are either for very special functions such as the examples
given in \citet{telgarsky2016benefits}, or special classes of functions
such as the ones having a specific compositional structure. For the
latter, the most notable are the results proved by \citet{poggio2017deep}
that the approximation error for deep neural network models is exponentially
better than the error for the shallow ones for a class of functions
with specific compositional structure. However, given a general function
$f$, one cannot calculate the distance from $f$ to such class of
functions. In the more general case, \citet{yarotsky2017error} considered
$C^{\beta}$-differentiable functions, and proved that the number
of parameters needed to achieve an error tolerance of $\varepsilon$
is $\mathcal{O}(\varepsilon^{-\nicefrac{d}{\beta}}\log\nicefrac{1}{\varepsilon})$.
\citet{montanelli2017deep} considered functions in Koborov space.
Using connection with sparse grids, they proved that the number of
parameters needed to achieve an error tolerance of $\varepsilon$
is $\mathcal{O}(\varepsilon^{-\nicefrac{1}{2}}(\log\nicefrac{1}{\varepsilon})^{d})$.

For shallow networks, there has been a long history of proving the
so-called universal approximation theorem, going back to the 1980s
\citep{cybenko1989approximation}. For networks with one hidden layer,
\citet{barron1993universal} established a convergence rate of $\mathcal{O}(n^{-\nicefrac{1}{2}})$
where $n$ is the number of hidden nodes. Such universal approximation
theorems can also be proved for deep networks. \citet{lu2017expressive}
considered networks of width $d+4$ for functions in $d$ dimension,
and proved that these networks can approximate any integrable function
with sufficient number of layers. However, they did not give the convergence
rate with respect to depth. To fill in this gap, we give a simple
proof that the same kind of convergence rate for shallow networks
can also be proved for deep networks.

The main purpose of this paper is to prove that for analytic functions,
deep neural network approximations converges exponentially fast. The
convergence rate deteriorates as a function of the dimensionality
of the problem. Therefore the present result is only of significance
in low dimension. However, this result does reveal a real superior
approximation property of the deep networks for a wide class of functions.

Specifically this paper contains the following contributions:
\begin{enumerate}
\item We construct neural networks with fixed width $d+4$ to approximate
a large class of functions where the convergence rate can be established.
\item For analytic functions, we obtain exponential convergence rate, i.e.
the depth needed only depends on $\log\nicefrac{1}{\varepsilon}$
instead of $\varepsilon$ itself.
\end{enumerate}

\section{The setup of the problem}

We begin by defining the network structure and the distance used in
this paper, and proving the corresponding properties for the addition
and composition operations.

We will use the following notations:
\begin{enumerate}
\item Colon notation for subscript: let $\{x_{m:n}\}=\{x_{i}:i=m,m+1,\dots,n\}$
and $\{x_{m_{1}:n_{1},m_{2}:n_{2}}\}=\{x_{i,j}:i=m_{1},m_{1}+1,\dots,n_{1},j=m_{2},m_{2}+1,\dots,n_{2}\}$.
\item Linear combination: denote $y\in\mathcal{L}(x_{1},\dots,x_{n})$ if
there exist $\beta_{i}\in\mathbb{R}$, $i=1,\dots,n$, such that $y=\beta_{0}+\beta_{1}x_{1}+\cdots+\beta_{n}x_{n}$.
\item Linear combination with ReLU activation: denote $\tilde{y}\in\tilde{\mathcal{L}}(x_{1},\dots,x_{n})$
if there exists $y\in\mathcal{L}(x_{1},\dots,x_{n})$ and $\tilde{y}=\mathrm{ReLU}(y)=\max(y,0)$. 
\end{enumerate}
\begin{defn}
Given a function $f(x_{1},\dots,x_{d})$, if there exist variables
$\{y_{1:L,1:M}\}$ such that
\begin{equation}
y_{1,m}\in\tilde{\mathcal{L}}(x_{1:d}),\quad y_{l+1,m}\in\tilde{\mathcal{L}}(x_{1:d},y_{l,1:M}),\quad f\in\mathcal{L}(x_{1:d},y_{1:L,1:M}),\label{eq:nn_class}
\end{equation}
where $m=1,\dots,M$, $l=1,\dots,L-1$, then $f$ is said to be in
the \emph{neural nets class} $\mathcal{F}_{L,M}\left(\mathbb{R}^{d}\right)$,
and $\{y_{1:L,1:M}\}$ is called a set of \emph{hidden variables}
of $f$.
\end{defn}

A function $f\in\mathcal{F}_{L,M}$ can be regarded as a neural net
with skip connections from the input layer to the hidden layers, and
from the hidden layers to the output layer. This representation is
slightly different from the one in standard fully-connected neural
networks where connections only exist between adjacent layers. However,
we can also easily represent such $f$ using a standard network without
skip connection.
\begin{prop}
\label{prop:wide2deep}A function $f\in\mathcal{F}_{L,M}\left(\mathbb{R}^{d}\right)$
can be represented by a ReLU network with depth $L+1$ and width $M+d+1$.
\end{prop}

\begin{proof}
Let $\{y_{1:L,1:M}\}$ be the hidden variables of $f$ that satisfies
(\ref{eq:nn_class}), where
\[
f=\alpha_{0}+\sum_{i=1}^{d}\alpha_{i}x_{i}+\sum_{l=1}^{L}\sum_{m=1}^{M}\beta_{l,m}y_{l,m}.
\]
Consider the following variables $\{h_{1:L,1:M}\}$:
\[
h_{l,1:M}=y_{l,1:M},\quad h_{l,M+1:M+d}=x_{1:d}
\]
 for $l=1,\dots,L$, and
\[
h_{1,M+d+1}=\alpha_{0}+\sum_{i=1}^{d}\alpha_{i}x_{i},\quad h_{l+1,M+d+1}=h_{l,M+d+1}+\sum_{m=1}^{M}\beta_{l,m}h_{l,m}
\]
for $l=1,\dots,,L-1$. One can see that $h_{1,m}\in\tilde{\mathcal{L}}(x_{1:d})$,
$h_{l+1,m}\in\tilde{\mathcal{L}}(h_{l,1:M+d+1})$, $m=1,\dots,M+d+1$,
$l=1,\dots,L-1$, and $f\in\mathcal{L}(h_{L,1:M+d+1})$, which is
a representation of a standard neural net.
\end{proof}
\begin{prop}
(Addition and composition of neural net class $\mathcal{F}_{L,M}$)
\begin{enumerate}
\item 
\begin{equation}
\mathcal{F}_{L_{1},M}+\mathcal{F}_{L_{2},M}\subseteq\mathcal{F}_{L_{1}+L_{2},M},\label{eq:nn_add}
\end{equation}
i.e. if $f_{1}\in\mathcal{F}_{L_{1,}M}\left(\mathbb{R}^{d}\right)$
and $f_{2}\in\mathcal{F}_{L_{2},M}\left(\mathbb{R}^{d}\right)$, then
$f_{1}+f_{2}\in\mathcal{F}_{L_{1}+L_{2},M}$. 
\item 
\begin{equation}
\mathcal{F}_{L_{2},M}\circ\mathcal{F}_{L_{1},M+1}\subseteq\mathcal{F}_{L_{1}+L_{2},M+1},\label{eq:nn_comp}
\end{equation}
i.e if $f_{1}(x_{1},\dots,x_{d})\in\mathcal{F}_{L_{1},M+1}\left(\mathbb{R}^{d}\right)$
and $f_{2}(y,x_{1},\dots,x_{d})\in\mathcal{F}_{L_{2},M}\left(\mathbb{R}^{d+1}\right)$,
then
\[
f_{2}(f_{1}(x_{1},\dots,x_{d}),x_{1},\dots,x_{d})\in\mathcal{F}_{L_{1}+L_{2},M+1}\left(\mathbb{R}^{d}\right).
\]
\end{enumerate}
\end{prop}

\begin{proof}
For the addition property, denote the hidden variables of $f_{1}$
and $f_{2}$ as $\{y_{1:L_{1},1:M}^{(1)}\}$ and $\{y_{1:L_{2},1:M}^{(2)}\}$.
Let
\[
y_{1:L_{1},1:M}=y_{1:L_{1},1:M}^{(1)},\quad y_{L_{1}+1:L_{1}+L_{2},1:M}=y_{1:L_{2},1:M}^{(2)}.
\]
By definition, $\{y_{1:L_{1}+L_{2},1:M}\}$ is a set of hidden variables
of $f_{1}+f_{2}$. Thus $f_{1}+f_{2}\in\mathcal{F}_{L_{1}+L_{2},M}$.

For the composition property, let the hidden variables of $f_{1}$
and $f_{2}$ as $\{y_{1:L_{1},1:M+1}^{(1)}\}$ and $\{y_{1:L_{2},1:M}^{(2)}\}$.
Let
\begin{gather*}
y_{1:L_{1},1:M+1}=y_{1:L_{1},1:M+1}^{(1)},\quad y_{L_{1}+1:L_{1}+L_{2},1:M}=y_{1:L_{2},1:M}^{(2)},\\
y_{L_{1}+1,M+1}=y_{L_{1}+2,M+1}=\cdots=y_{L_{1}+L_{2},M+1}=f_{1}(x_{1},\dots,x_{d}).
\end{gather*}
One can see that $\{y_{1:L_{1}+L_{2},1:M+1}\}$ is a set of hidden
variables of $f_{2}(f_{1}(\boldsymbol{x}),\boldsymbol{x})$, thus
the composition property (\ref{eq:nn_comp}) holds.
\end{proof}
\begin{defn}
Given a continuous function $\varphi(\boldsymbol{x})$, $\boldsymbol{x}\in[-1,1]^{d}$
and a continuous function class $\mathcal{F}\left([-1,1]^{d}\right)$,
define the \emph{$L_{\infty}$ distance}\global\long\def\dist{\mathrm{dist}}
\begin{equation}
\dist\left(\varphi,\mathcal{F}\right)=\inf_{f\in\mathcal{F}}\max_{\boldsymbol{x}\in[-1,1]^{d}}|\varphi(\boldsymbol{x})-f(\boldsymbol{x})|.
\end{equation}
\end{defn}

\begin{prop}
(Addition and composition properties for distance function)
\begin{enumerate}
\item Let $\varphi_{1}$ and $\varphi_{2}$ be continuous functions. Let
$\mathcal{F}_{1}$ and $\mathcal{F}_{2}$ be two continuous function
classes, then
\begin{equation}
\dist\left(\alpha_{1}\varphi_{1}+\alpha_{2}\varphi_{2},\mathcal{F}_{1}+\mathcal{F}_{2}\right)\leq|\alpha_{1}|\dist\left(\varphi_{1},\mathcal{F}_{1}\right)+|\alpha_{2}|\dist\left(\varphi_{2},\mathcal{F}_{2}\right),\label{eq:dist_add}
\end{equation}
where $\alpha_{1}$ and $\alpha_{2}$ are two real numbers.
\item Assume that $\varphi_{1}(\boldsymbol{x})=\varphi_{1}(x_{1},\dots,x_{d})$,
$\varphi_{2}(y,\boldsymbol{x})=\varphi_{2}(y,x_{1},\dots,x_{d})$
satisfy $\varphi_{1}\left([-1,1]^{d}\right)\subseteq[-1,1]$. Let
$\mathcal{F}_{1}\left([-1,1]^{d}\right)$, $\mathcal{F}_{2}\left([-1,1]^{d+1}\right)$
be two continuous function classes, then
\begin{equation}
\dist\left(\varphi_{2}(\varphi_{1}(\boldsymbol{x}),\boldsymbol{x}),\mathcal{F}_{2}\circ\mathcal{F}_{1}\right)\leq L_{\varphi_{2}}\dist\left(\varphi_{1},\mathcal{F}_{1}\right)+\dist\left(\varphi_{2},\mathcal{F}_{2}\right),\label{eq:dist_comp}
\end{equation}
where $L_{\varphi_{2}}$ is the Lipschitz norm of $\varphi_{2}$ with
respect to $y$.
\end{enumerate}
\end{prop}

\begin{proof}
The additional property obviously holds. Now we prove the composition
property. For any $f_{1}\in\mathcal{F}_{1}$, $f_{2}\in\mathcal{F}_{2}$,
one has
\begin{align*}
|\varphi_{2}(\varphi_{1}(\boldsymbol{x}),\boldsymbol{x})-f_{2}(f_{1}(\boldsymbol{x}),\boldsymbol{x})| & \leq|\varphi_{2}(\varphi_{1}(\boldsymbol{x}),\boldsymbol{x})-\varphi_{2}(f_{1}(\boldsymbol{x}),\boldsymbol{x})|+|\varphi_{2}(f_{1}(\boldsymbol{x}),\boldsymbol{x})-f_{2}(f_{1}(\boldsymbol{x}),\boldsymbol{x})|\\
 & \leq L_{\varphi_{2}}\|\varphi_{1}(\boldsymbol{x})-f_{1}(\boldsymbol{x})\|_{\infty}+\|\varphi_{2}(y,\boldsymbol{x})-f_{2}(y,\boldsymbol{x})\|_{\infty}.
\end{align*}
Take \global\long\def\argmin{\mathrm{argmin}}
$f_{1}^{\star}=\argmin_{f}\|\varphi_{1}(\boldsymbol{x})-f(\boldsymbol{x})\|_{\infty}$
and $f_{2}^{\star}=\argmin_{f}\|\varphi_{2}(y,\boldsymbol{x})-f(y,\boldsymbol{x})\|_{\infty}$,
then
\[
|\varphi_{2}(\varphi_{1}(\boldsymbol{x}),\boldsymbol{x})-f_{2}^{\star}(f_{1}^{\star}(\boldsymbol{x}),\boldsymbol{x})|\le L_{\varphi_{2}}\dist\left(\varphi_{1},\mathcal{F}_{1}\right)+\dist\left(\varphi_{2},\mathcal{F}_{2}\right),
\]
thus the composition property (\ref{eq:dist_comp}) holds.
\end{proof}
Now we are ready to state the main theorem for the approximation of
analytic functions.
\begin{thm}
\label{thm:analytic}Let $f$ be an analytic function over $(-1,1)^{d}$.
Assume that the power series $f(\boldsymbol{x})=\sum_{\boldsymbol{k}\in\mathbb{N}^{d}}a_{\boldsymbol{k}}\boldsymbol{x^{k}}$
is absolutely convergent in $[-1,1]^{d}$. Then for any $\delta>0$,
there exists a function $\hat{f}$ that can be represented by a deep
ReLU network with depth $L$ and width $d+4$, such that

\begin{equation}
|f(\boldsymbol{x})-\hat{f}(\boldsymbol{x})|<2\sum_{\boldsymbol{k}\in\mathbb{N}^{d}}|a_{\boldsymbol{k}}|\cdot\exp\left(-d\delta\left(e^{-1}L^{\nicefrac{1}{2d}}-1\right)\right)
\end{equation}
 for all $\boldsymbol{x}\in[-1+\delta,1-\delta]^{d}$.
\end{thm}

\section{Proof}

The construction of $\hat{f}$ is motivated by the approximation of
the square function $\varphi(x)=x^{2}$ and multiplication function
$\varphi(x,y)=xy$ proposed in \citet{yarotsky2017error}, \citet{liang2016deep}.
We use this as the basic building block to construct approximations
to monomials, polynomials, and analytic functions.
\begin{lem}
\label{lem:square}The function $\varphi(x)=x^{2}$, $x\in[-1,1]$
can be approximated by deep neural nets with an exponential convergence
rate:
\begin{equation}
\dist\left(\varphi=x^{2},\mathcal{F}_{L,2}\right)\le2^{-2L}.
\end{equation}
\end{lem}

\begin{proof}
Consider the function
\[
g(y)=\begin{cases}
2y, & 0\le y<1/2,\\
2(1-y), & 1/2\le y\le1,
\end{cases}
\]
then $g(y)=y-4\mathrm{ReLU}(y-1/2)$ in $[0,1]$. Define the hidden
variables $\{y_{1:L,1:2}\}$ as follows:
\[
\begin{array}{ll}
y_{1,1}=\mathrm{ReLU}(x), & y_{1,2}=\mathrm{ReLU}(-x),\\
y_{2,1}=\mathrm{ReLU}(y_{1,1}+y_{1,2}), & y_{2,2}=\mathrm{ReLU}(y_{1,1}+y_{1,2}-1/2),\\
y_{l+1,1}=\mathrm{ReLU}(2y_{l,1}-4y_{l,2}), & y_{l+1,2}=\mathrm{ReLU}(2y_{l,1}-4y_{l,2}-1/2)
\end{array}
\]
for $l=2,3,\dots,L-1$. Using induction, one can see that $|x|=y_{1,1}+y_{1,2}$
and $g_{l}(|x|)=\underbrace{g\circ g\circ\cdots\circ g}_{l}(|x|)=2y_{l+1,1}-4y_{l+1,2}$,
$l=1,\dots,L-1$ for $x\in[-1,1]$. Now let
\[
f(x)=|x|-\sum_{l=1}^{L-1}\frac{g_{l}(|x|)}{2^{2l}},
\]
then $f\in\mathcal{F}_{L,2}$, and $\left|x^{2}-f(x)\right|\le2^{-2L}$
for $x\in[-1,1]$.
\end{proof}
\begin{lem}
\label{lem:multiply}For multiplication function $\varphi(x,y)=xy$,
we have
\begin{equation}
\dist\left(\varphi=xy,\mathcal{F}_{3L,2}\right)\le3\cdot2^{-2L}.\label{eq:multiply}
\end{equation}
\end{lem}

\begin{proof}
Notice that
\[
\varphi=xy=2\left(\frac{x+y}{2}\right)^{2}-\frac{1}{2}x^{2}-\frac{1}{2}y^{2}.
\]
Applying the addition properties (\ref{eq:nn_add})(\ref{eq:dist_add})
and \lemref{square}, we obtain (\ref{eq:multiply}).
\end{proof}
Now we use the multiplication function as the basic block to construct
monomials and polynomials.
\begin{lem}
\label{lem:monomial}For a monomial $M_{p}(\boldsymbol{x})$ of $d$
variables with degree $p$, we have
\begin{equation}
\dist\left(M_{p},\mathcal{F}_{3(p-1)L,3}\right)\le3(p-1)\cdot2^{-2L}.
\end{equation}
\end{lem}

\begin{proof}
Let $M_{p}(\boldsymbol{x})=x_{i_{1}}x_{i_{2}}\cdots x_{i_{p}}$, $i_{1},\dots,i_{p}\in\{1,\dots,d\}$.
Using induction, assume that the lemma holds for the degree-$p$ monomial
$M_{p}$, consider a degree-$(p+1)$ monomial $M_{p+1}(\boldsymbol{x})=M_{p}(\boldsymbol{x})\cdot x_{i_{p+1}}$.
Let $\varphi(y,x)=yx$, then $M_{p+1}(\boldsymbol{x})=\varphi(M_{p}(\boldsymbol{x}),x_{i_{p+1}})$.
From composition properties (\ref{eq:nn_comp})(\ref{eq:dist_comp})
and \lemref{multiply}, we have
\begin{multline*}
\dist\left(M_{p+1},\mathcal{F}_{3pL,3}\right)\le\dist\left(\varphi(M_{p}(\boldsymbol{x}),x_{i_{p+1}}),\mathcal{F}_{3L,2}\circ\mathcal{F}_{3(p-1)L,3}\right)\\
\le L_{\varphi}\dist\left(M_{p},\mathcal{F}_{3(p-1)L,3}\right)+\dist\left(\varphi,\mathcal{F}_{3L,2}\right)\le3p\cdot2^{-2L}.
\end{multline*}
Note that the Lipschitz norm $L_{\varphi}=1$ since $x_{i_{p+1}}\in[-1,1]$.
\end{proof}
\begin{lem}
\label{lem:polynomial}For a degree-$p$ polynomial $P_{p}(\boldsymbol{x})=\sum_{|\boldsymbol{k}|\le p}a_{\boldsymbol{k}}\boldsymbol{x^{k}}$,
$\boldsymbol{x}\in[-1,1]^{d}$, $\boldsymbol{k}=(k_{1},\dots,k_{d})\in\mathbb{N}^{d}$,
we have
\begin{equation}
\dist\left(P_{p},\mathcal{F}_{\binom{p+d}{d}(p-1)L,3}\right)<3(p-1)\cdot2^{-2L}\sum_{|\boldsymbol{k}|\le p}|a_{\boldsymbol{k}}|.
\end{equation}
\end{lem}

\begin{proof}
The lemma can be proved by applying the addition property (\ref{eq:nn_add})(\ref{eq:dist_add})
and \lemref{monomial}. 

Note that the number of monomials of $d$ variables with degree less
or equal to $p$ is $\binom{p+d}{d}$.
\end{proof}
Now we are ready to prove \thmref{analytic}.
\begin{proof}
Let $\varepsilon=\exp\left(-d\delta\left(e^{-1}L^{\nicefrac{1}{2d}}-1\right)\right)$,
then $L=\left[e\left(\frac{1}{d\delta}\log\frac{1}{\varepsilon}+1\right)\right]^{2d}$.
Without loss of generality, assume $\sum_{\boldsymbol{k}}|a_{\boldsymbol{k}}|=1$.
We will show that there exists $\hat{f}\in\mathcal{F}_{L,3}$ such
that $\|f-\hat{f}\|_{\infty}<2\varepsilon$.

Denote
\[
f(\boldsymbol{x})=P_{p}(\boldsymbol{x})+R(\boldsymbol{x})\coloneqq\sum_{|\boldsymbol{k}|\le p}a_{\boldsymbol{k}}\boldsymbol{x^{k}}+\sum_{|\boldsymbol{k}|>p}a_{\boldsymbol{k}}\boldsymbol{x^{k}}.
\]
For $\boldsymbol{x}\in[-1+\delta,1-\delta]^{d}$, we have $|R(\boldsymbol{x})|<(1-\delta)^{p}$,
thus truncation to $p=\frac{1}{\delta}\log\frac{1}{\varepsilon}$
will ensure $|R(\boldsymbol{x})|<\varepsilon$.

From \lemref{polynomial}, we have $\dist\left(P_{p},\mathcal{F}_{L,3}\right)<3(p-1)\cdot2^{-2L'}$,
where
\begin{align*}
L' & =L\binom{p+d}{p}^{-1}(p-1)^{-1}>L\left[\frac{(p+d)^{d}}{(d/e)^{d}}\right]^{-1}p^{-1}\\
 & =L\left[e\left(\frac{1}{d\delta}\log\frac{1}{\varepsilon}+1\right)\right]^{-d}\left(\frac{1}{\delta}\log\frac{1}{\varepsilon}\right)^{-1}=\left[e\left(\frac{1}{d\delta}\log\frac{1}{\varepsilon}+1\right)\right]^{d}\left(\frac{1}{\delta}\log\frac{1}{\varepsilon}\right)^{-1}\\
 & \gg\log\frac{1}{\varepsilon}+\log\frac{1}{\delta}
\end{align*}
for $d\ge2$ and $\varepsilon\ll1$, then $\dist\left(P_{p},\mathcal{F}_{L,3}\right)<3(p-1)\cdot2^{-2L'}\ll\varepsilon$.
Thus there exists $\hat{f}\in\mathcal{F}_{L,3}$ such that $\|P_{p}-\hat{f}\|_{\infty}<\varepsilon$,
and $\|f-\hat{f}\|_{\infty}\leq\|f-P_{p}\|_{\infty}+\|P_{p}-\hat{f}\|_{\infty}<2\varepsilon$.
\end{proof}
One can also formulate \thmref{analytic} as follows:
\begin{cor}
Assume that the analytic function $f(\boldsymbol{x})=\sum_{\boldsymbol{k}\in\mathbb{N}^{d}}a_{\boldsymbol{k}}\boldsymbol{x^{k}}$
is absolutely convergent in $[-1,1]^{d}$. Then for any $\varepsilon,\delta>0$,
there exists a function $\hat{f}$ that can be represented by a deep
ReLU network with depth $L=\left[e\left(\frac{1}{d\delta}\log\frac{1}{\varepsilon}+1\right)\right]^{2d}$
and width $d+4$, such that $|f(\boldsymbol{x})-\hat{f}(\boldsymbol{x})|<2\varepsilon\sum_{\boldsymbol{k}}|a_{\boldsymbol{k}}|$
for all $\boldsymbol{x}\in[-1+\delta,1-\delta]^{d}$.
\end{cor}

\section{The convergence rate for the general case}

Here we prove that for deep neural networks, the approximation error
decays like $\mathcal{O}\left((N/d)^{-\nicefrac{1}{2}}\right)$ where $N$ is
the number of parameters in the model. The proof is quite simple but
the result does not seem to be available in the existing literature.
\begin{thm}
\label{thm:barron_deep}Given a function $f:\mathbb{R}^{d}\to\mathbb{R}$
with Fourier representation
\[
f(\boldsymbol{x})=\int_{\mathbb{R}^{d}}e^{\mathrm{i}\boldsymbol{\omega}\cdot\boldsymbol{x}}\hat{f}(\boldsymbol{\omega})\mathrm{d}\boldsymbol{\omega},
\]
and a compact domain $B\subset\mathbb{R}^{d}$ containing 0, let
\[
C_{f,B}=\int_{B}|\boldsymbol{\omega}|_{B}|\hat{f}(\boldsymbol{\omega})|\mathrm{d}\boldsymbol{\omega},
\]
where $|\boldsymbol{\omega}|_{B}=\sup_{\boldsymbol{x}\in B}|\boldsymbol{\omega}\cdot\boldsymbol{x}|$.
Then there exists a ReLU network $f_{L,M}$ with width $M+d+1$ and
depth $L$, such that
\begin{equation}
\int_{B}|f(\boldsymbol{x})-f_{L,M}(\boldsymbol{x})|^{2}\mathrm{d}\mu(\boldsymbol{x})\le\frac{8C_{f,B}^{2}}{ML},
\end{equation}
where $\mu$ is an arbitrary probability measure.
\end{thm}

Here the number of parameters 
\[
N=(d+1)(M+d+1)+(M+d+2)(M+d+2)(L-1)+(M+d+2)=\mathcal{O}\left((M+d)^{2}L\right).
\]
Taking $M=d$, we will have $L=\mathcal{O}\left(N/d^{2}\right)$ and the convergence
rate becomes $\mathcal{O}\left((ML)^{-\nicefrac{1}{2}}\right)=\mathcal{O}\left((N/d)^{-\nicefrac{1}{2}}\right)$.
Note that in the universal approximation theorem for  shallow networks
with one hidden layer , one can prove the same convergence rate $\mathcal{O}(n^{-\nicefrac{1}{2}})=\mathcal{O}((N/d)^{-\nicefrac{1}{2}})$.
Here $n$ is the number of hidden nodes and $N=(d+2)n+1$ is the number
of parameters.

Theorem \ref{thm:barron_deep} is a direct consequence of the following
theorem by \citet{barron1993universal} for networks with one hidden
layer and sigmoidal type of activation function. Here a function $\sigma$
is \emph{sigmoidal} if it is bounded measurable on $\mathbb{R}$ with
$\sigma(+\infty)=1$ and $\sigma(-\infty)=0$.
\begin{thm}
\label{thm:barron}Given a function $f$ and a domain $B$ such that
$C_{f,B}$ finite, given a sigmoidal function $\sigma$, there exists
a linear combination
\[
f_{n}(\boldsymbol{x})=\sum_{j=1}^{n}c_{j}\sigma(\boldsymbol{a}_{j}\cdot\boldsymbol{x}+b_{j})+c_{0},\quad\boldsymbol{a}_{j}\in\mathbb{R}^{d},\ b_{j},c_{j}\in\mathbb{R},
\]
 such that 
\begin{equation}
\int_{B}|f(\boldsymbol{x})-f_{n}(\boldsymbol{x})|^{2}\mathrm{d}\mu(\boldsymbol{x})\le\frac{4C_{f,B}^{2}}{n}.
\end{equation}
\end{thm}

Notice that
\[
\sigma(z)=\mathrm{ReLU}(z)-\mathrm{ReLU}(z-1)
\]
is sigmoidal, so we have:
\begin{cor}
\label{cor:barron}Given a function $f$ and a set $B$ with $C_{f,B}$
finite, there exists a linear combination of $n$ ReLU functions 
\[
f_{n}(\boldsymbol{x})=\sum_{j=1}^{n}c_{j}\mathrm{ReLU}(\boldsymbol{a}_{j}\cdot\boldsymbol{x}+b_{j})+c_{0},
\]
 such that 
\[
\int_{B}|f(\boldsymbol{x})-f_{n}(\boldsymbol{x})|^{2}\mathrm{d}\mu(\boldsymbol{x})\le\frac{8C_{f,B}^{2}}{n}.
\]
\end{cor}

Next we convert this shallow network to a deep one.
\begin{lem}
\label{lem:wide2deep}Let $f_{n}:\mathbb{R}^{d}\to\mathbb{R}$ be
a ReLU network with one hidden layer (as shown in the previous corollary).
For any decomposition $n=m_{1}+\cdots+m_{L}$, $n_{k}\in\mathbb{N}^{\star}$,
$f_{n}$ can also be represented by a ReLU network with $L$ hidden
layers, where the $l$-th layer has $m_{l}+d+1$ nodes.
\end{lem}

\begin{proof}
Denote the input by $\boldsymbol{x}=(x_{1},\dots,x_{d})$. We construct
a network with $L$ hidden layers in which the $l$-th layer has $m_{l}+d+1$
nodes $\{h_{l,1:m_{l}+d+1}\}$. Similar to the construction in proposition
\propref{wide2deep}, let
\[
h_{L,1:d}=h_{L-1,1:d}=\cdots=h_{1,1:d}=x_{1:d},\quad h_{l,d+j}=\mathrm{ReLU}\left(\boldsymbol{a}_{l,j}\cdot\boldsymbol{x}+b_{l,j}\right)
\]
for $j=1,\dots,m_{l}$, $l=1,\dots,L$, and
\[
h_{1,d+m_{1}+1}=c_{0},\quad h_{l+1,d+m_{l+1}+1}=h_{l,d+m_{l}+1}+\sum_{j=1}^{m_{l}}c_{l,j}h_{l,d+j}
\]
for $l=1,\dots,L-1$. Here we use the notation $\boldsymbol{a}_{l,j}=\boldsymbol{a}_{m_{1}+\cdots+m_{l-1}+j}$
(the same for $b_{l,j}$ and $c_{l,j}$). One can see that $h_{1,m}\in\hat{\mathcal{L}}(x_{1:d})$,
$h_{l+1,m}\in\hat{\mathcal{L}}(h_{l,1:d+m_{l}+1})$, $m=1,\dots,d+m_{l}+1$,
$l=1,\dots,L-1$ and
\[
h_{l,d+m_{l}+1}=c_{0}+\sum_{j=1}^{m_{1}+\cdots+m_{l-1}}c_{j}\mathrm{ReLU}(\boldsymbol{a}_{j}\cdot\boldsymbol{x}+b_{j}).
\]
Thus
\[
f_{n}=h_{L,d+m_{L}+1}+\sum_{j=1}^{m_{L}}c_{L,j}h_{L,d+j}\in\mathcal{L}(h_{L,1:d+m_{L}+1})
\]
can be represented by this deep network.
\end{proof}
Now consider a network with $L$ layers where each layer has the same
width $M+d+1$. From \lemref{wide2deep}, this network is equivalent
to a one-layer network with $ML$ hidden nodes. Apply corollary \corref{barron},
we obtain the desired approximation result for deep networks stated
in \thmref{barron_deep}.

\textbf{Acknowledgement.} We are grateful to Chao Ma for very helpful
discussions during the early stage of this work. We are also grateful
to Jinchao Xu for his interest, which motivated us to write up this
paper. The work is supported in part by ONR grant N00014-13-1-0338
and Major Program of NNSFC under grant 91130005.



\bibliographystyle{plainnat}
\bibliography{approx_analytic}

\end{document}